\documentclass[lettersize,journal]{IEEEtran}
\usepackage{amsmath,amsfonts}
\usepackage{algorithmic}
\usepackage{algorithm}
\usepackage{array}
\usepackage[caption=false,font=normalsize,labelfont=sf,textfont=sf]{subfig}
\usepackage{textcomp}
\usepackage{stfloats}
\usepackage{url}
\usepackage{verbatim}
\usepackage{graphicx}
\usepackage{cite}
\usepackage{amsthm}
\usepackage{times}
\usepackage{graphicx}

\newtheorem{theorem}{Theorem}
\usepackage[numbers]{natbib}
\usepackage{multicol}
\usepackage[bookmarks=true]{hyperref}
\usepackage{amssymb}
\usepackage{booktabs} 
\usepackage{multirow} 
\usepackage{amsmath}  
\usepackage{graphicx}
\usepackage{algorithm}
\usepackage{algorithmic}
\usepackage{bm}
\usepackage{tabularx, booktabs, multirow}

\usepackage{xcolor}
\definecolor{commentblue}{RGB}{40, 116, 166} 
\definecolor{keywordpurple}{RGB}{125, 60, 152} 
\definecolor{darkblue}{RGB}{30, 100, 170} 

\usepackage{tcolorbox}
\tcbuselibrary{listings, skins}
\usepackage{listings}
\usepackage[table]{xcolor}
\usepackage{array}
\usepackage{ragged2e}
\usepackage{xcolor}

\definecolor{promptorange}{RGB}{255,165,0}
\definecolor{outputblue}{RGB}{214,232,245}
\definecolor{reqbrown}{RGB}{160,140,120}
\usepackage{multirow}
\usepackage{booktabs}
\usepackage[table]{xcolor}
\usepackage{pifont}

\newcolumntype{P}{>{\RaggedRight\arraybackslash}p{\dimexpr\linewidth-2\tabcolsep-2\arrayrulewidth}}

\begin{document}

\title{Beyond Pixels: Learning Invariant Rewards  \\ for Real-World Robotics From a Few Demonstrations}

\author{
  \IEEEauthorblockN{
    Tengye Xu$^{1,2}$,\ 
    Yangting Sun$^{2}$,\ 
    Ziju Shen$^{2,4}$,\ 
    Guanqi Chen$^{1}$,\ 
    Zhen Fu$^{2,3}$,\ 
    Yizhou Chen$^{1}$,\ 
    Hua Chen$^{2,5\dagger}$,\ 
    Jia Pan$^{1\dagger}$
  }
  \IEEEauthorblockA{
    $^{1}$School of Computing and Data Science, The University of Hong Kong \quad
    $^{2}$LimX Dynamics Technology Co., Ltd \\
    $^{3}$Southern University of Science and Technology \quad
    $^{4}$Peking University \quad
    $^{5}$Zhejiang University \\[2pt]
  }
}

\maketitle

\begin{abstract}
Designing reward functions that generalize beyond controlled laboratory settings remains a fundamental challenge in reinforcement learning for robotics. In open-world manipulation problems, a single task can appear in numerous variants through different object instances, positions, and camera viewpoints. Recent vision-based reward models tend to memorize specific
pixel distributions and fail to generalize beyond their training conditions.
To address this, we propose a framework that learns invariant symbolic reward functions from as few as five demonstrations.
The insight is to shift from visual feature-fitting to the discovery of behavioral invariants: task-level properties that remain constant across diverse visual instantiations.
The framework has two coupled components: a structural reward formulation that encodes task-level strategies and physical constraints while preserving optimal policy invariance, and a hybrid symbolic-numerical procedure that distills these invariants from demonstrations without online interaction. 
Experiments on eight Meta-World tasks and three Franka manipulation tasks demonstrate  that our method achieves stronger process alignment and policy rollout ranking abilities compared to  baselines, accelerating downstream policy learning. 
Three real-world out-of-distribution experiments further show that the same learned reward generalizes zero-shot to position, viewpoint, and object variations, enabling a single reward representation to be reused across diverse task
variants in practice.

\end{abstract}


\section{Introduction}
Developing robust, dense reward functions for open-world manipulation is a long-standing goal in reinforcement learning and Embodied AI. 
In an  open-world setting, a unit task like opening a box is rarely visually identical twice, due to the inherent variability of object instances, spatial configurations, and camera perspectives.
While the semantic essence of the task typically remains unchanged, 
the aforementioned inherent variability makes handcrafted reward functions brittle and impractical to scale. Rather than hand-crafting rewards, modern approaches turn to learning reward functions automatically from demonstrations. 
The key requirement is generalization: a learned reward model should distill the behavioral invariants of a task class from only a handful of demonstrations, enabling it to generalize beyond the limited visual conditions seen during training to the infinite variability encountered in the real world.

The recent success of using Vision-Language Models (VLMs) in open-world understanding has motivated their use for automatic reward design. One line of work uses general-purpose VLMs to score task progress 
directly from text descriptions and images~ \cite{rocamonde2023vision, ma2024vision, kim2025subtask}.
While these models offer broad 
semantic coverage, their scores lack temporal consistency and 
numerical calibration, limiting their reliability as dense rewards 
for policy optimization.

A second line 
~\citep{sontakke2023roboclip,kim2025subtask,hung2024victor,ye2023reinforcement,ayalew2025progressor,chen2026topreward,liang2026robometer,lee2026roboreward}
trains dedicated reward models on large-scale robotic 
datasets, with 
representative examples including LIV \citep{ma2023liv}, VLC \citep{alakuijala2024video}, ReWiND \citep{zhang2025rewind}, and SARM \citep{chen2025sarm}.
These models produce more accurate scores but are 
prone to visual overfitting, usually requiring in-distribution test tasks, and fail to correctly score the images, which are semantically identical but visually different.

\begin{figure}[t]
\begin{center}

\includegraphics[width=1\linewidth]{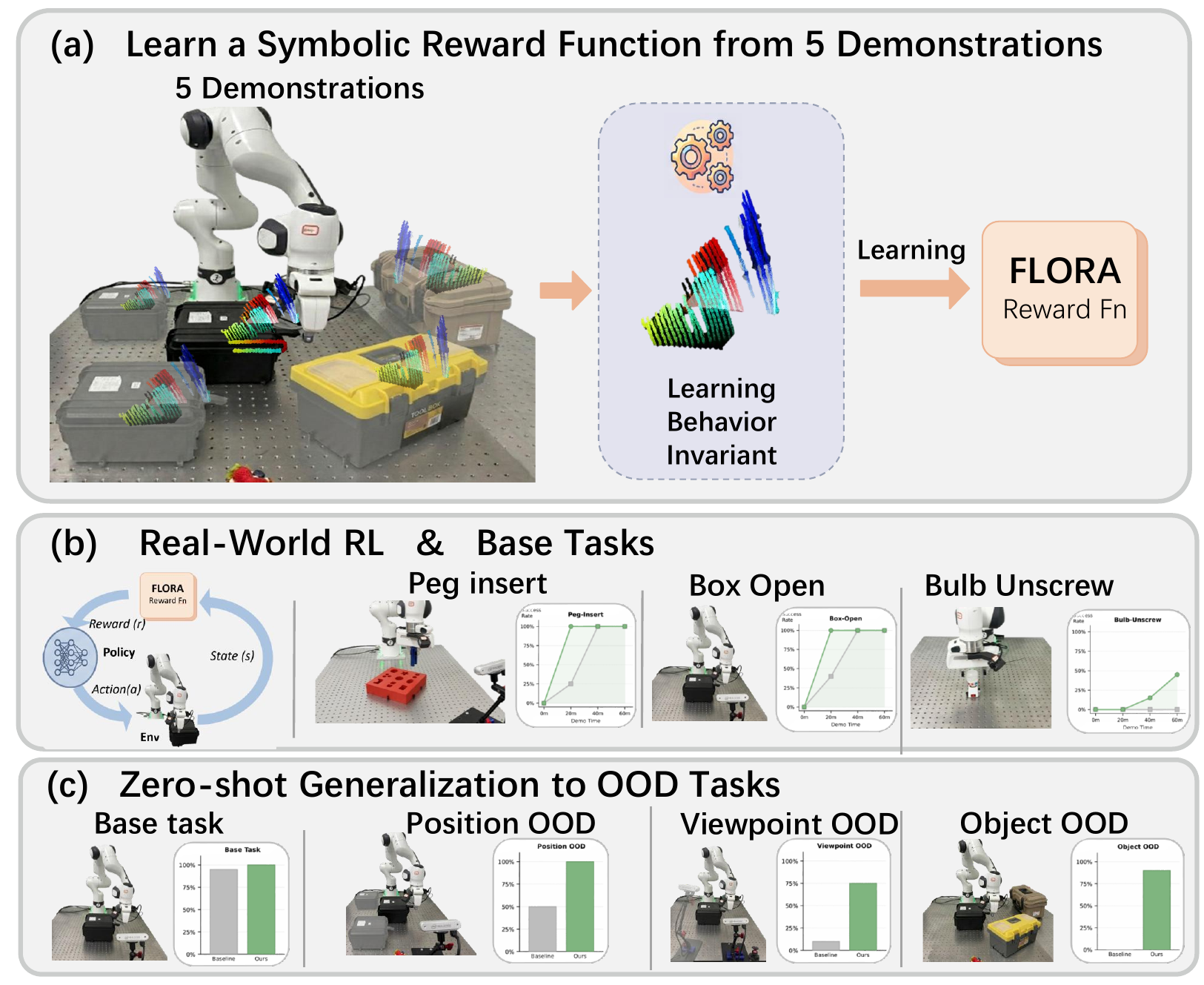}
\end{center}
\vspace{-0.5cm}
\caption{
Overview.
(a)~Five demonstrations are used to learn a symbolic reward function
that captures the behavioral invariants of the task rather than its
visual appearance.
(b)~The learned reward model guides real-world RL on a Franka arm
across three manipulation tasks: Peg-Insert, Box-Open, and
Bulb-Unscrew.
(c)~Without any retraining, the same reward model transfers zero-shot
to position, viewpoint, and object OOD variants.
\label{fig:overview flow generation}}
\vspace{-0.3cm}
\end{figure}

A more principled way to capture behavioral invariants is through programmatic reward functions. Unlike black-box neural networks, code-based rewards can explicitly represent physical laws and logical constraints that hold across task variants. Recent efforts, such as Eureka \citep{ma2023eureka} and Text2Reward \citep{xie2023text2reward}, have utilized Large Language Models (LLMs) to automatically synthesize reward functions. However, these methods generally depend on privileged state information, such as ground-truth 6D object poses, and  usually require extensive online interaction within simulation. These requirements are often infeasible in real-world settings, where access to an oracle state is  unavailable and interaction budgets are strictly limited by hardware safety and time.

Furthermore, existing VLM-based and programmatic methods often lack theoretical guarantees regarding optimal policy invariance.  This can lead to the phenomenon of reward hacking, where an agent maximizes the reward through unintended or suboptimal behaviors. For example, in a pick-and-place task, an agent may learn to hover infinitely near a target region rather than completing the placement. A robust reward framework must provide dense guidance without altering the task's original optimal policy.

Consequently, how to extract reward functions that are invariant to visual appearance, grounded in task logic, and provably policy-invariant, using only a small number of offline demonstrations without privileged state access, remains an open and critical challenge.

We address these challenges by introducing Flow-based Language-driven Offline Reward Adaptation (FLORA), a framework designed to learn invariant symbolic reward functions from as few as five demonstrations. 
To eliminate the dependency on privileged state information, we develop a Flow-Generator that extracts task-relevant motion flows from raw images to serve as a robust state representation.  To ensure theoretical rigor and prevent reward hacking, we adopt a Potential-Based Reward Shaping (PBRS) formulation. Within this framework, our synthesized programmatic functions serve as \textit{potential functions}; the final reward signal is derived via a PBRS-style post-processing module, which mathematically guarantees optimal policy invariance. With all these together, FLORA introduces a new theoretically grounded reward formulation for real-world Reinforcement Learning (RL), which can explicitly  encode task-level strategies and physical constraints from raw visual inputs.

Crucially, to automatically learn this reward model in a low-data regime, we propose a hybrid symbolic-numerical learning framework. This method combines discrete symbolic search with continuous parameter optimization, allowing FLORA to discover robust, invariant reward functions from as few as five demonstrations.
In summary,  our contributions are as follows:
\begin{itemize}
  \item We introduce {FLORA}, a novel framework for learning invariant symbolic
    reward functions from as few as five visual demonstrations,
    without privileged state access or online environment interaction.
 \item We propose a structural reward formulation that 
 decomposes the
    reward model into a Flow-Generator, a symbolic potential function,
    and a PBRS-MS module. The formulation carries a formal guarantee
    of optimal policy invariance,
    eliminating reward hacking by construction.
 
  \item We develop a hybrid symbolic-numerical learning procedure that
  jointly optimizes the discrete program structure via LLM reflection
    and the continuous parameters via Bayesian Optimization, using only
    five  demonstrations.
    \item   Extensive {simulation and real-world experiments}
    show that FLORA outperforms vision-language reward baselines
    in process alignment and policy learning efficiency, and transfers
    zero-shot to position, viewpoint, and object variations.

\end{itemize}

\section{Related Work}

\subsection{Classical Reward Learning}
Classical approaches to reward learning include inverse reinforcement learning (IRL) 
\citep{ng2000algorithms,arora2021survey,abbeel2004apprenticeship,ziebart2008maximum,finn2016guided}
, which infers rewards from expert demonstrations, and preference-based reward learning \citep{hejna2023few,christiano2017deep,sadigh2017active}, which uses human comparisons to align agent behavior with user intent. IRL provides a principled formulation but often depends on predefined reward features, strong modeling assumptions, or a large number of demonstrations. Preference-based methods reduce the need for explicit reward design, but they usually require repeated human feedback during training or adaptation. These requirements make both approaches difficult to deploy in the low-data, real-robot setting, where only a few visual demonstrations are available and additional interaction is costly.

\subsection{Vision-Language Reward Models}
Prior work has also explored using VLMs to generate reward signals \citep{rocamonde2023vision, ma2024vision, kim2025subtask} or utilizing VLMs to provide preference rankings  between image pairs \citep{wang2024rl,venkataraman2024real,ghosh2025preference}. Although VLMs provide broad semantic priors, their scores are not optimized to be temporally smooth or numerically calibrated across consecutive robot states. This makes them useful for coarse semantic judgment but less reliable as dense rewards for policy optimization, where small inconsistencies can change the gradient of exploration.
Another line of work \citep{sontakke2023roboclip,kim2025subtask,hung2024victor,ye2023reinforcement,ayalew2025progressor,ma2023liv,alakuijala2024video,zhang2025rewind,chen2025sarm} trains models on  large-scale robotic demonstration datasets. While these models provide more accurate scores, they are  prone to visual overfitting, failing to correctly score scenes that are semantically identical but visually different. Furthermore, these methods typically impose a heavy data burden; for example, SARM \citep{chen2025sarm} requires over 200 hours of demonstrations for tasks like T-shirt folding, while ReWiND \citep{zhang2025rewind} necessitates demonstrations across 20 target-environment tasks to ensure coverage. 

A parallel line of work \citep{lee2026roboreward, liang2026robometer} trains general-purpose reward models on large-scale datasets containing both successful and failure trajectories. While these models achieve strong discrimination on in-distribution tasks, they rely on large-scale pretraining and task-specific LoRA fine-tuning before deployment, which introduces a nontrivial computational cost (e.g., ~8 hours on an NVIDIA RTX A6000).
FLORA targets a setting where only a handful of demonstrations are available and the reward model must generalize zero-shot to unseen task variants without any further training. Furthermore, RoboReward  \citep{lee2026roboreward} produces discrete reward signals, whereas our framework provides dense, continuous shaping signals grounded in PBRS theory, a property critical for guiding exploration in contact-rich manipulation.

\subsection{Programmatic Reward Functions}
Recent methods \citep{hu2023language,yu2023language,ma2023eureka,xie2023text2reward,ma2024dreureka,zeng2024video2reward} use LLMs or VLMs to synthesize executable reward code, encoding task-level heuristics and physical priors directly in programs. However, they typically assume access to privileged state information and often require extensive online interaction in simulation to refine the generated rewards. These assumptions are difficult to satisfy in real-world settings, where oracle states are unavailable and interaction budgets are tightly constrained. 
FLORA shares the goal of obtaining interpretable programmatic rewards, but differs in the source of state information and the optimization setting: it learns symbolic potential functions over motion-flow representations rather than privileged simulator variables, and optimizes their structure and parameters entirely offline from a small demonstration set.

\subsection{Motion Flow in Manipulation}
Motion flow has emerged as a robust representation for cross-domain \citep{wen2023any,xu2024flow,yuan2024general,zhi20253dflowaction,xu2025a0} 
transfer. By abstracting raw images into object-centric trajectories, motion flow can provide a practical visual proxy for task-relevant state variables when privileged object poses are unavailable. While some recent approaches like HUDOR \citep{guzeyhudor},  GENFLOWRL \citep{yu2025genflowrl} utilize motion flow for reward design, they typically rely on {handcrafted functions} that require human expertise and limit scalability across task variants. In contrast, our work  proposes  a hybrid symbolic-numerical learning framework that automatically learns symbolic reward functions from a few demonstrations and guarantees optimal policy invariance.

\begin{figure*}[t]
\begin{center}
\includegraphics[width=0.92\linewidth]{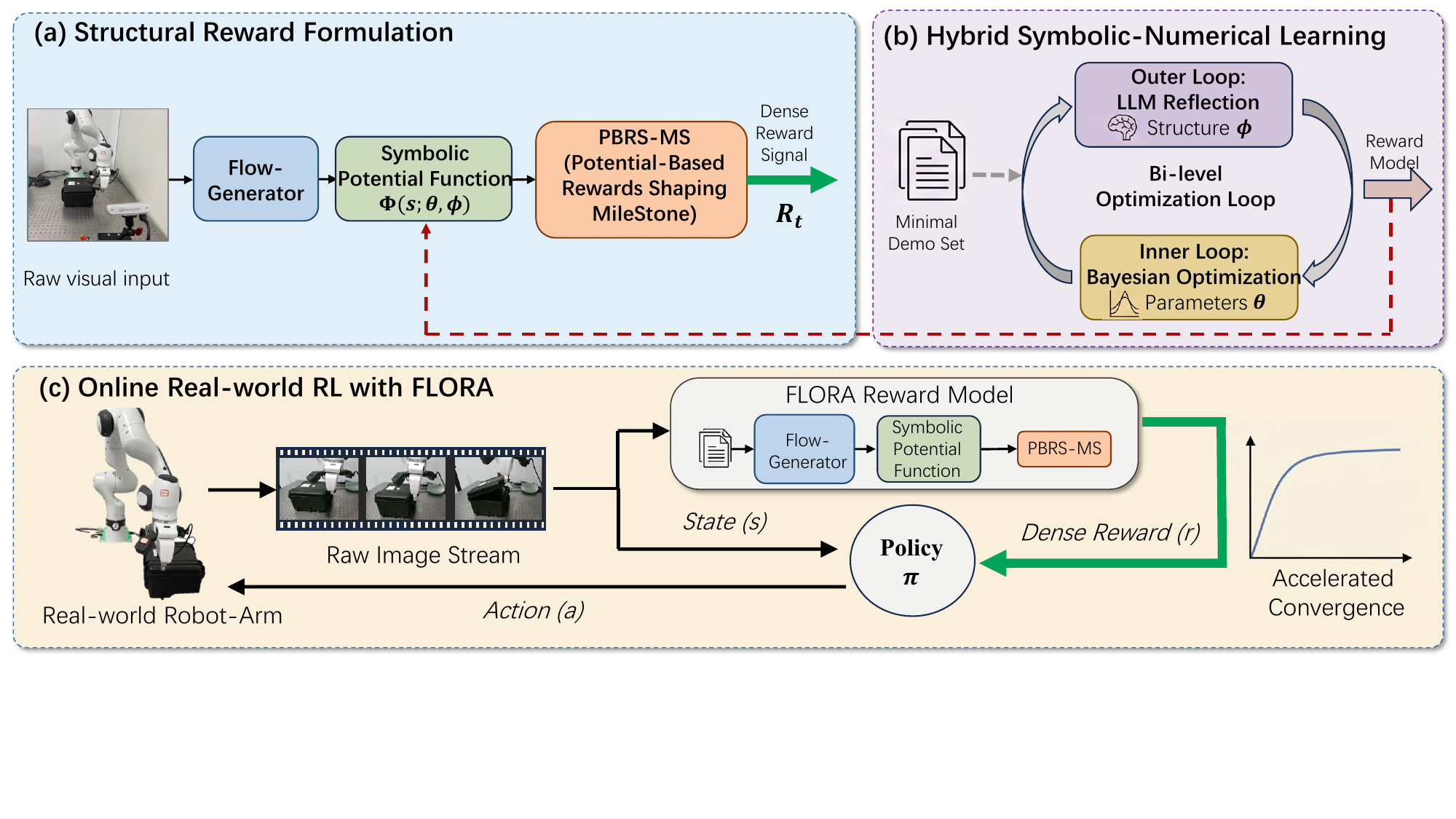}
\end{center}
\vspace{-2.3cm}
\caption{Framework Overview: 
(a) \textbf{Structural Reward Formulation}: A pipeline that maps raw visual input to robust reward signals. It comprises a Flow-Generator for flow generation, a symbolic potential function for  progress estimation, and a Potential-Based Reward Shaping-MileStone (PBRS-MS) module  to ensure optimal policy invariance and signal stability. 
(b) \textbf{Hybrid Symbolic-Numerical Learning}: A bi-level optimization loop that jointly discovers the optimal symbolic program structure $\phi$ (Outer Loop via LLM reflection) and its continuous parameters $\theta$ (Inner Loop via Bayesian Optimization) using a minimal demonstration set. 
(c) \textbf{Online Real-world RL}: Deployment of the synthesized reward model to provide dense guidance. The framework computes dense rewards from raw image streams  to accelerate policy convergence in complex manipulation tasks.}
\label{fig: Overal framework}
\vspace{-0.2cm}
\end{figure*}

\subsection{Optimality-Preserving Reward Shaping}

Reward shaping aims to accelerate policy learning by providing dense, informative signals. Among various methods, Potential-Based Reward Shaping (PBRS) \citep{ng1999policy} is a  widely used  framework due to its theoretical guarantee of optimal policy invariance, and has been  used in robotics tasks \citep{ye2023reinforcement,yin2025rapidly}. However, traditional PBRS does not provide a mechanism for designing task-specific potential functions
and suffers from potential
collapse in long-horizon tasks.
FLORA learns the potential
function automatically from demonstrations and introduces a
milestone-augmented formulation that provably resolves collapse
while preserving the invariance guarantee.

\section{Background and problem formulation}

\subsection{Reinforcement Learning and PBRS}
A  manipulation task is modeled as a Markov Decision Process (MDP)
$\mathcal{M} = (\mathcal{S}, \mathcal{A}, P, R, \gamma)$,
where $\mathcal{S}$ and $\mathcal{A}$ are the state and action spaces,
$P: \mathcal{S}\times\mathcal{A}\to\mathcal{S}$ the transition dynamics,
$R: \mathcal{S}\times\mathcal{A}\times\mathcal{S}\to\mathbb{R}$ the reward,
and $\gamma\in(0,1]$ the discount factor.
The agent seeks a policy $\pi$ maximizing
$G = \mathbb{E}[\sum_{t=0}^{\infty}\gamma^t R_t]$.
In real-world manipulation, we  address the sparse reward setting: the environment returns a binary
signal $R\in\{0,1\}$, which provides little gradient for exploration.

To mitigate reward sparsity without introducing reward hacking, we employ
PBRS~\citep{ng1999policy}.
A shaped reward $R'(s_t,a_t,s_{t+1}) = R(s_t,a_t,s_{t+1})
+ \gamma\phi(s_{t+1}) - \phi(s_t)$
augments $R$ with a potential function $\phi:\mathcal{S}\to\mathbb{R}$.
PBRS guarantees \emph{optimal policy invariance}: any optimal policy for
the shaped MDP $\mathcal{M}' = (\mathcal{S},\mathcal{A},R',P,\gamma)$
is also optimal for $\mathcal{M}$, so the shaping signal accelerates
convergence without altering the intended objective.

\subsection{Problem Formulation}
\label{sec: problem formulation}
Given a sparse-reward MDP and a minimal set of visual demonstrations $\mathcal{D}_{\text{demo}}$,  the goal  is to learn a symbolic reward function $R'$  that provides dense, reliable reward signals from raw visual observations. 
To be effective for open-world manipulation, $R'$ must
satisfy two constraints: \emph{behavioral invariance}, meaning the reward
is grounded in the task's logical structure rather than absolute pixel
values to enable generalization to new task variants; and \emph{optimality invariance}, meaning $R'$ strictly preserves
the optimal policy of $\mathcal{M}$, preventing reward hacking.

\section{Structural Reward Formulation}
\label{sec:methodology}

Rather than learning a reward function end-to-end from images, 
FLORA encodes the Behavioral and Optimality Invariance constraints 
directly into the reward architecture, converting a constrained 
optimization problem into a tractable unconstrained one.

The formulation has three components, as illustrated in 
Fig.~ \ref{fig: Overal framework} (a).
A \emph{Flow-Generator} abstracts raw visual input into 
object-centric motion flows, restricting the potential function's 
input space to task-relevant spatial relations and satisfying 
behavioral invariance by design.
A \emph{Symbolic Potential Function} maps these flows to scalar 
progress values through programmatic logic that encodes task 
strategies and physical constraints.
A \emph{PBRS-MS module} wraps the potential within a 
potential-difference reward structure, stabilizing training against 
potential collapse while formally guaranteeing optimal policy invariance.
This section details these three components; 
Section~\ref{sec:method_discovery} then describes a {Hybrid Symbolic-Numerical Learning Framework}, which dictates \textit{how} the potential function is  learned from as few as five demonstrations.

\begin{figure}[t]
\begin{center}

\includegraphics[width=0.95\linewidth]{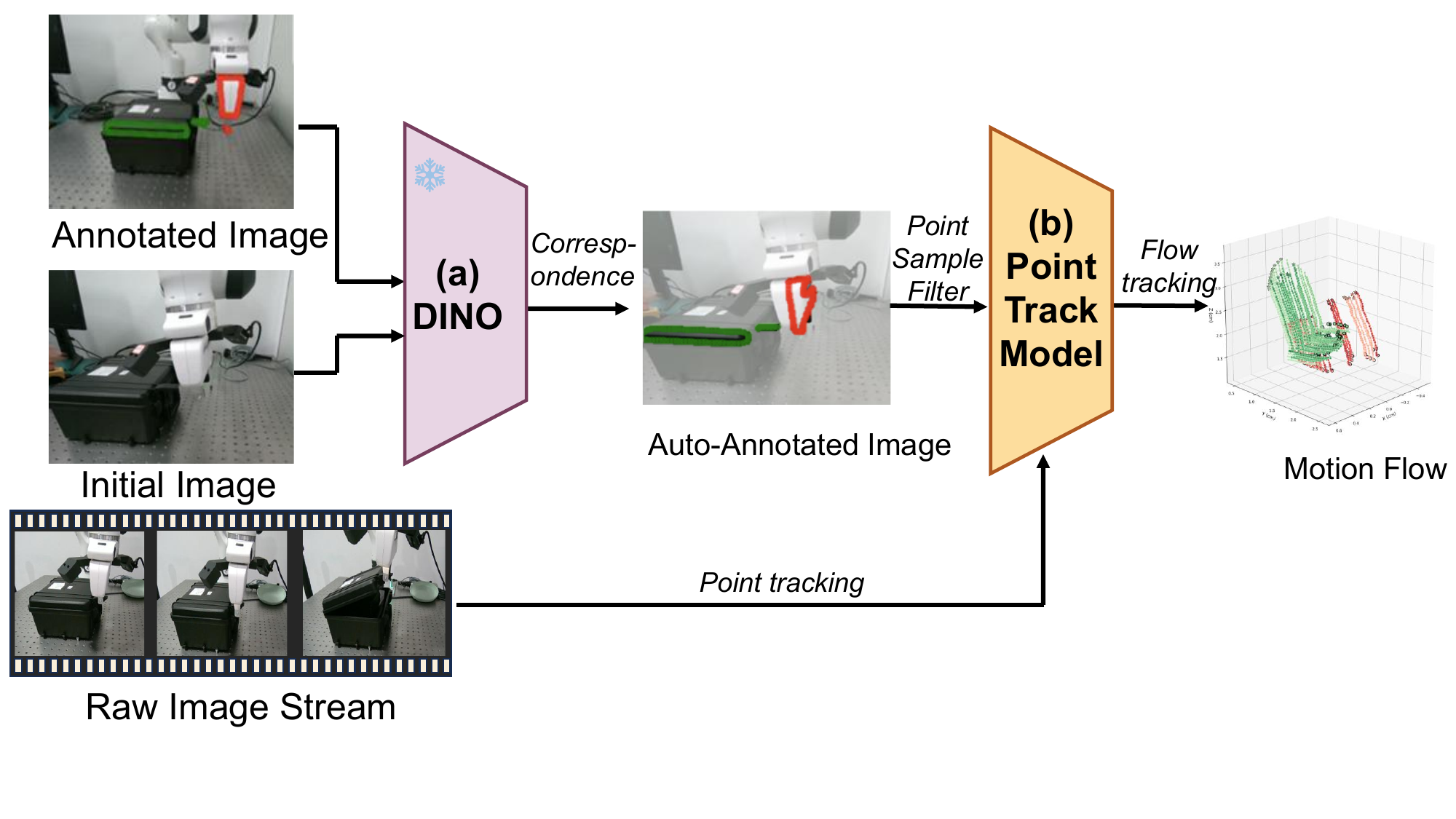}
\end{center}
\vspace{-0.95cm}
\caption{The flow generation procedure.\label{fig:flow generation}}
\vspace{-0.3cm}
\end{figure}

\subsection{Flow-Generator: Grounding via Motion Flow}
\label{subsection:Flow-Generator}
To satisfy the \textit{behavioral invariance} constraint, the reward model must be grounded in task-relevant relative spatial relationships between objects (e.g., the alignment between a gripper and a handle) rather than raw pixel intensities. We propose a {Flow-Generator} that transforms high-dimensional visual observations into a task-relevant motion flow representation $o_t$. This ensures that the downstream potential function  perceives the task’s logical state regardless of visual distractors or scene variations. 
As shown in Fig.~\ref{fig:flow generation}, the pipeline operates in three stages.
Regions of Interest (RoIs), such as  the robot gripper, the manipulated object,
and the target,  are identified once per task class in a reference image, either
manually or via Grounding DINO~\citep{liu2023grounding}.
Since re-running the detector at every episode is costly and brittle under
appearance shift, we use the dense correspondence model
DINOv2~\citep{oquab2023dinov2} to re-locate these RoIs in the initial frame
of each new episode.
A 3D point tracker (e.g., TAPIP3D~\citep{zhang2025tapip3d}) then tracks
keypoints within each RoI throughout the episode; all tracked points are then
expressed in a canonical object-centric coordinate frame to achieve spatial invariance.

\subsection{Symbolic Potential Function}

With the input space canonicalized by the Flow-Generator,  we define the core of our reward model as a symbolic potential function $\phi$, parameterized by $\theta$. For each timestep $t$, this function maps the current motion flow $o_t$ and the initial motion flow $o_0$ to a scalar potential $p_t$ and an auxiliary task stage $\hat{\mu}_t$:
\begin{equation}
(p_t, \hat{\mu}_t) = \phi(o_t, o_0; \theta).
\end{equation}
The potential function $\phi$ is decoupled into two components: its symbolic structure and numerical parameters.
The \emph{symbolic structure} $\phi \in \Phi$ is the programmatic logic of the
reward, encoding semantic reward structure and task-level strategy; it is
automatically discovered by an LLM (Section~\ref{sec:method_discovery}).
The \emph{numerical parameters} $\theta \in \Theta$ are the continuous
coefficients embedded within the code, such as  thresholds, weighting factors,
and gain coefficients, and are optimized by Bayesian Optimization
(Section~\ref{sec:method_discovery}).

\subsection{PBRS-MS: Theoretical Guarantees for Manipulation}

To satisfy the Optimality Invariance constraint, we wrap the symbolic potential within a PBRS form. However,  directly applying PBRS in high-dimensional manipulation tasks can lead to training instabilities due to potential collapse.
In manipulation, near-miss states (e.g., dropping an object at the final step) are adjacent to success states.  Due to the telescoping property of PBRS, the cumulative shaped return over $N$ steps is dominated by the boundary potentials $ \phi(s_N)$:
$G = \sum_{t=0}^{N-1} \gamma^t (R_t + \gamma \phi(s_{t+1}) - \phi(s_t)) = \sum_{t=0}^{N-1} \gamma^t R_t + \gamma^N \phi(s_N) - \phi(s_0).$
Consequently, a trajectory that lifts an object but drops it at the terminal step yields a return nearly identical to one that never initiated a grasp, as both terminate in low-potential states $s_N$. This lack of differentiation prevents the $Q$-function from correctly assigning credit to partial successes, leading to severe training instability. Furthermore, local basins and plateaus in $\phi$ can generate uninformative gradients that trap exploration, a failure mode particularly prevalent in long-horizon manipulation.

To provide a robust global signal that resists  potential collapse, we augment PBRS with a milestone-based progress tracker. The key idea is to lock in partial progress so that late-stage failures do not erase earlier achievements.

\paragraph{Milestone Definition} 
We partition the potential range $[0, \phi_{\max}]$ into $K$ discrete levels using uniformly spaced thresholds $\kappa_k = \frac{k}{K}$ for $k \in \{1, \dots, K\}$. A state $s$ is said to have \emph{reached} milestone $k$ if its potential satisfies $\phi(s) \ge \kappa_k \cdot \phi_{\max}$.

We
augment the state space with a monotonic milestone index  $m_t \in \{0, 1, \dots, K\}$ that records the highest milestone achieved up to time $t$:
$    m_t = \max\left( m_{t-1},\; \max\{ k \mid \phi(s_t) \ge \kappa_k \cdot \phi_{\max} \} \right), \quad m_0 = 0.$  Thus we define a {Shaped Augmented MDP} $\tilde{M}' = (\tilde{\mathcal{S}}, \mathcal{A}, \tilde{P}, \tilde{R}', \gamma)$:  The state space is $\tilde{\mathcal{S}} = \mathcal{S} \times \{0, \dots, K\}$.     The transition dynamics $\tilde{P}$ are defined as:
    $ \tilde{P}(s', m' | s, m, a) = P(s' | s, a) \cdot \mathbb{P}(m' | s', m) $
    where the milestone transition $\mathbb{P}(m' | s', m)$ is deterministic.

\paragraph{Global Potential and Shaped reward} We define a cumulative potential over milestones: $\Psi(m) = \sum_{k=1}^{m} \hat{R}_k$, where $\hat{R}_k > 0$ is the reward bonus for reaching milestone $k$. The final shaped reward integrates both local and global signals:
\begin{equation}
    R_t' = R_t + \underbrace{\gamma \phi(s_{t+1}) - \phi(s_t)}_{\text{Local PBRS}} + \underbrace{\gamma \Psi(m_{t+1}) - \Psi(m_t)}_{\text{Global MS}}.
    \label{eq:final_shaped_reward}
\end{equation}

Since both $\phi$ and $\Psi$ are 
defined over this augmented space via potential-based shaping, 
the following theorem holds:

\begin{figure}[t]
\begin{center}
\includegraphics[width=0.8\linewidth]{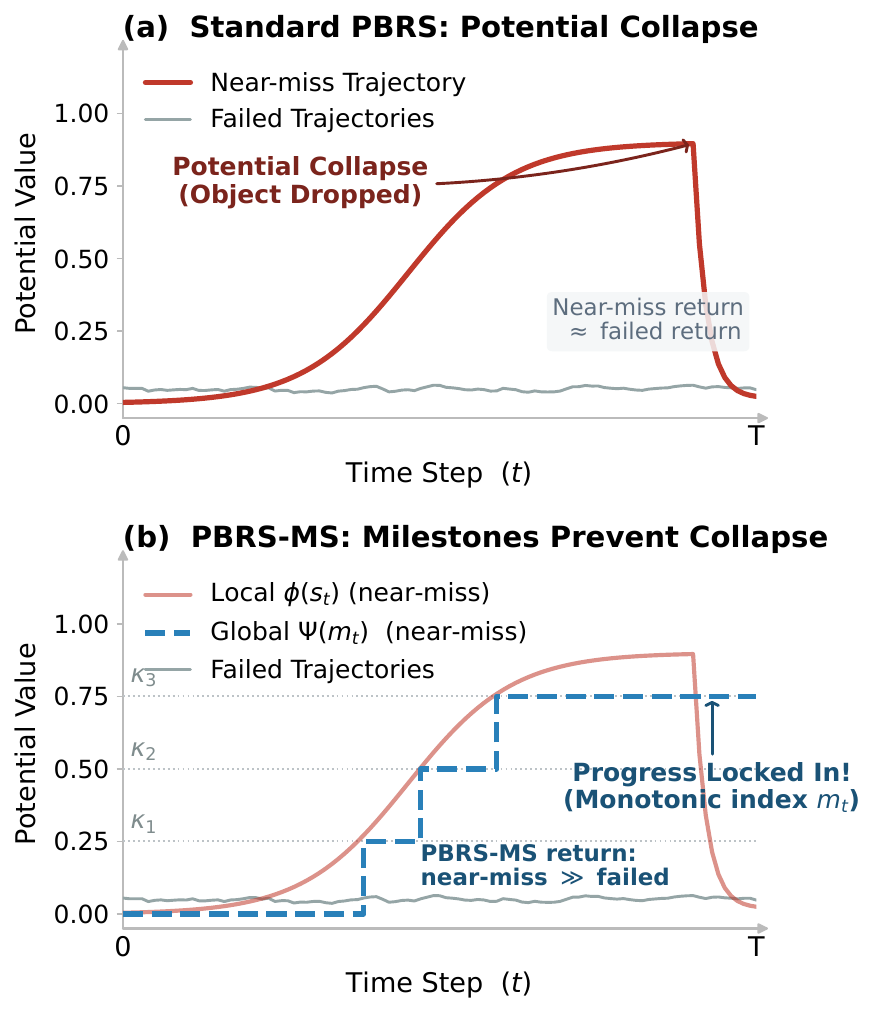}
\end{center}
\vspace{-0.6cm}
\caption{Potential collapse under standard PBRS and its resolution by our proposed PBRS-MS.
}
\label{fig: pbms}
\vspace{-0.3cm}
\end{figure}

\begin{theorem}[Policy Invariance of PBRS-MS]
For any discount factor $\gamma \in [0, 1)$, any optimal policy $\tilde{\pi}^*$ in the shaped augmented MDP $\tilde{M}'$ is also an optimal policy in the original MDP $M$.
\end{theorem}

\begin{proof}[Proof Sketch]
Let $\tilde{M}$ denote the unshaped counterpart of $\tilde{M}'$, i.e., $\tilde{M}'$
with $F_{\mathrm{MS}}\!=\!0$.
The proof follows a two-stage reduction $\tilde{M}' \to \tilde{M} \to M$.
\textit{Stage~1.} Setting $\Phi(\tilde{s})\!=\!\phi(s)+\Psi(m)$, the shaping
signal satisfies $F_{\mathrm{MS}} = \gamma\Phi(\tilde{s}_{t+1})-\Phi(\tilde{s}_t)$,
which is PBRS form~\citep{ng1999policy}.
Hence $\tilde{Q}^*_{\tilde{M}'}$ and $\tilde{Q}^*_{\tilde{M}}$ differ only by
$\Phi(\tilde{s})$, leaving the $\arg\max_a$ unchanged; every optimal policy in
$\tilde{M}'$ is optimal in $\tilde{M}$.
\textit{Stage~2.} Since $\tilde{R}$ ignores $m$ and
$\sum_{m'}\mathbb{P}(m'|s',m)=1$, the candidate
$\tilde{V}(s,m)\!=\!V_M^*(s)$ is a fixed point of the Bellman operator
$\tilde{\mathcal{T}}$.
Uniqueness of the $\gamma$-contraction gives
$\tilde{V}^*_{\tilde{M}}(s,m)\!=\!V_M^*(s)$, so $\tilde{\pi}^*(s,m)=\pi^*(s)$
independently of $m$.
Combining both stages completes the proof.
(\textit{Full proof in Supplementary Material~I.})
\end{proof}

\section{Hybrid Symbolic-Numerical Learning of Potential Functions}
\label{sec:method_discovery}

The structural reward formulation in Sec.~\ref{sec:methodology} transforms the
original optimization problem into a tractable, unconstrained one.
 We solve this using a bi-level optimization strategy over a demonstration dataset $\mathcal{D}_{\text{demo}}$. The goal is to find a symbolic structure $\phi$ and a parameter vector $\theta$ that maximize a surrogate performance objective $\mathcal{J}$ while ensuring generalization to unseen trajectories. The pipeline is shown in  Fig. \ref{fig: hybrid opt}.
\subsection{Problem Formulation}
Given a demonstration set partitioned into training and validation subsets, $\mathcal{D}_{\text{demo}} = \{ \mathcal{D}_{\text{train}}, \mathcal{D}_{\text{val}} \}$, we seek the optimal potential function $(\phi, \theta )$ by solving:
\begin{align}
\phi^* &= \arg\max_{\phi \in \Phi} \mathcal{J}(\phi, \theta^*(\phi); \mathcal{D}_{\text{demo}}) \\
\text{} \quad \theta^*(\phi) &= \arg\max_{\theta \in \Theta} \mathcal{J}( \theta; \phi,\mathcal{D}{_\text{train}})
\end{align}
where $\mathcal{J}$ is the multi-objective surrogate (Section~\ref{subsec: the Surrogate Objective}). In this formulation, the inner-loop optimization solves for the best numerical alignment $\theta$ for a given logic structure on $\mathcal{D}_{\text{train}}$, while the outer-loop  searches for a symbolic structure $\phi$ that generalizes  to both training set and held-out validation set.

\subsection{Optimization Dataset Preparation}
To ground the optimization in semantic task logic rather than visual features,
each demonstration in $\mathcal{D}_{\text{demo}}$ is preprocessed in three steps
before being passed to the bi-level optimizer.
Raw observations are first mapped to motion-flow representations following
Section~\ref{subsection:Flow-Generator}, replacing pixel-level inputs with
object-centric spatial relations.
We then use GPT-4.1~\citep{achiam2023gpt} to decompose each task into a sequence
of discrete symbolic subtasks and let GPT-4.1 annotate every trajectory with the
corresponding stage labels $\mu_t$; shuffling~\citep{ma2024vision} and majority
voting are applied to reduce annotation noise.
Finally, $\mathcal{D}_{\text{demo}}$ is split 60/40 into training and validation
subsets $\mathcal{D}_{\text{train}}$ and $\mathcal{D}_{\text{val}}$.
Crucially, the inner loop only observes $\mathcal{D}_{\text{train}}$, so the
outer loop is forced to select symbolic structures that generalize to held-out
trajectories rather than memorizing the training set.

\begin{algorithm}[t]
\caption{Hybrid Symbolic-Numerical Learning}
\label{alg:hybrid_offline_learning}
\begin{algorithmic}[1]
\REQUIRE Task $l$, LLM $\mathcal{L}$, Surrogate $\mathcal{J}$,
         Data $\{\mathcal{D}_{\text{train}}, \mathcal{D}_{\text{val}}\}$,
         Iterations $N$, Batch size $K$, Top-$m$ size $m$, Prompt $p_0$
\STATE $p \gets p_0$;\quad $v_{\text{best}} \gets -\infty$
\FOR{$i = 1$ \TO $N$}
    \STATE $\{\phi_k, \theta_k\}_{k=1}^K \sim \mathcal{L}(l, p)$
           \hfill \textit{// symbolic outer-loop}
    \FOR{$k = 1$ \TO $K$}
        \STATE $\theta_k^* \gets \textsc{BO}(\phi_k, \theta_k, \mathcal{J},
               \mathcal{D}_{\text{train}})$
               \hfill \textit{// numerical inner-loop}
        \STATE $v_k \gets \mathcal{J}(\phi_k, \theta_k^*;\,
               \mathcal{D}_{\text{demo}})$
    \ENDFOR
    \STATE $\mathcal{D}_{\text{top}} \gets
           \textsc{Top-}m\!\left(\{(\phi_k, \theta_k^*, v_k)\}_{k=1}^K\right)$
    \IF{$\max(v_k) > v_{\text{best}}$}
        \STATE $(\phi_{\text{best}}, \theta_{\text{best}}, v_{\text{best}})
               \gets \arg\max_{(\phi,\theta,v)\,\in\,\mathcal{D}_{\text{top}}} v$
    \ENDIF
    \STATE $p \gets p \oplus \textsc{Reflection}(\mathcal{D}_{\text{top}})$
\ENDFOR
\RETURN $\phi_{\text{best}},\, \theta_{\text{best}}$
\end{algorithmic}
\end{algorithm}

\subsection{The Surrogate Objective ($\mathcal{J}$)}
\label{subsec: the Surrogate Objective}
We define a multi-objective surrogate $\mathcal{J}(\phi, \theta; \mathcal{D}{_\text{train})} $ to measure the alignment of the learned potential function with  an ideal one. We seek to maximize:
\begin{equation}
\mathcal{J}(\phi, \theta; \mathcal{D}{_\text{train})} = \sum_{\tau \in \mathcal{D}_{\text{train}}} \left[ \lambda_1 \mathcal{C}_{\text{stage}} + \lambda_2 \mathcal{C}_{\text{prog}} + \lambda_3 \mathcal{C}_{\text{pbrs}} \right]
\end{equation}
where the components are defined as:

\paragraph{Stage alignment ($\mathcal{C}_{\text{stage}}$)}
This measures  the accuracy of the predicted subtasks $\hat{\mu}_t$ against VLM labels $\mu_t$:
\begin{equation}
\mathcal{C}_{\text{stage}}(\phi)
= \frac{1}{T} \sum_{t=1}^{T} \mathbf{1}\{\hat{\mu}_t = \mu_t\}.
\end{equation}
where $\mathbf{1}\{\cdot\}$ is the indicator function. A high stage-alignment score indicates that the symbolic structure has
correctly encoded the task's sequential logic.

\paragraph{Progress monotonicity ($\mathcal{C}_{\text{prog}}$)}
This  measures 
the Spearman correlation between $\phi(o_t)$ and normalized time  $t/T$:
\begin{equation}
\mathcal{C}_{\text{prog}}(\phi) = \mathrm{Corr}(\phi(o_t),\, t/T).
\end{equation}
Spearman correlation is preferred over Pearson because monotonicity, not
linearity, is desired for a progress proxy.
\paragraph{PBRS positivity ($\mathcal{C}_{\text{pbrs}}$)}
This term measures the frequency of non-negative shaping signals:
\begin{equation}
\mathcal{C}_{\text{pbrs}}(\phi)
= \frac{1}{T-1} \sum_{t=1}^{T-1}
  \mathbf{1}\{\gamma\phi(o_{t+1}) - \phi(o_t) \ge 0\}.
\end{equation}

\begin{figure}[t]
\begin{center}
\includegraphics[width=1\linewidth]{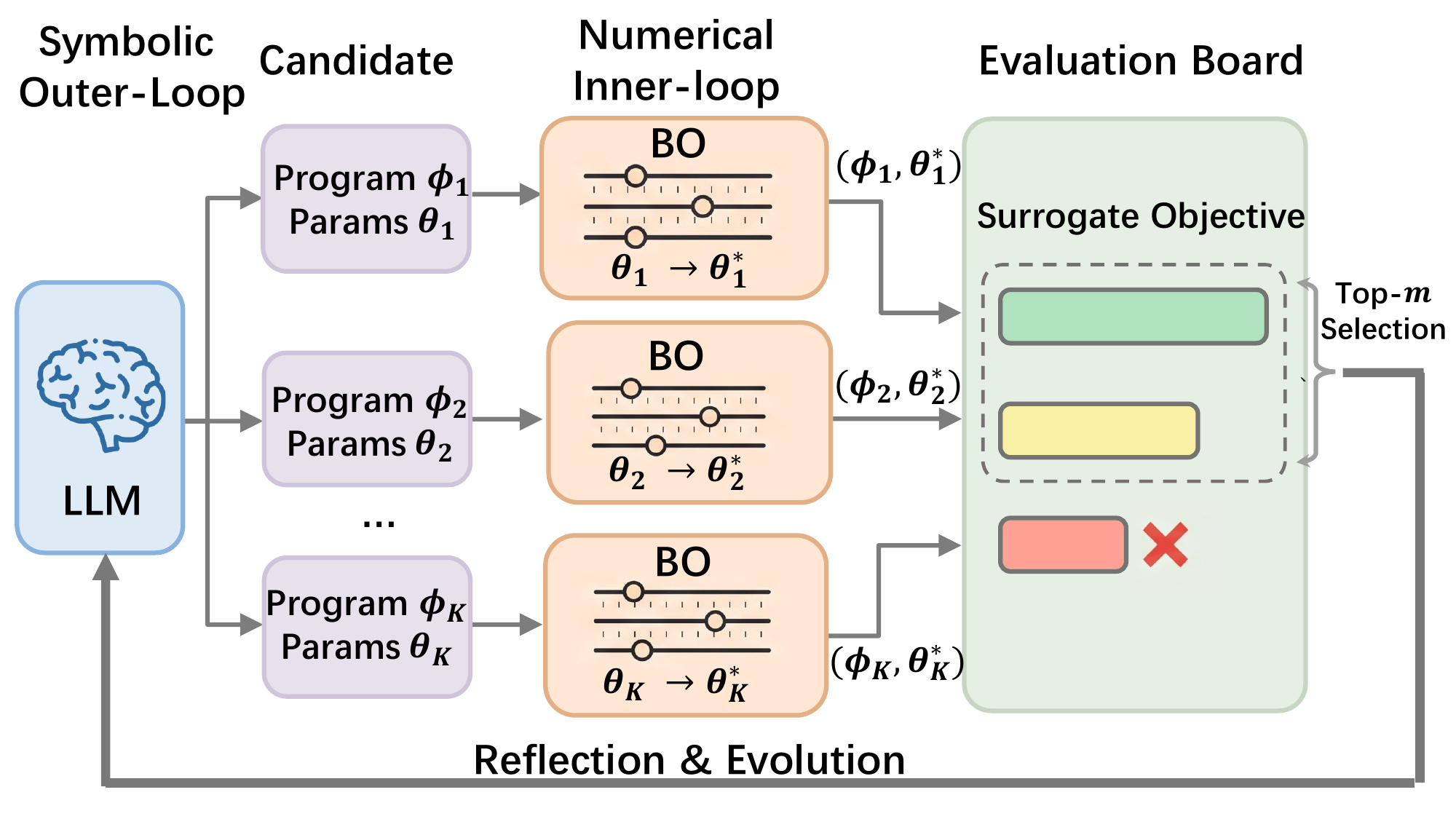}
\end{center}
\vspace{-0.6cm}
\caption{The Hybrid Symbolic-Numerical Learning Pipeline.
}
\label{fig: hybrid opt}
\vspace{-0.3cm}
\end{figure}

\subsection{Bi-Level Optimization Procedure}
The nested optimization is solved through a hybrid search strategy summarized
in Algorithm~\ref{alg:hybrid_offline_learning}.
\paragraph{Numerical inner loop}
For each candidate symbolic structure $\phi_k$, the lower-level problem is solved
with Bayesian Optimization~\citep{frazier2018tutorial} using the Upper Confidence
Bound (UCB) acquisition function~\citep{pmlr-v22-kaufmann12}:
\begin{equation}
\theta_k^* = \arg\max_{\theta_k}\;
\mathcal{J}(\theta_k;\, \phi_k, \mathcal{D}_{\text{train}}).
\end{equation}
The initial parameters are drawn directly from the LLM, which ensures they
lie within a physically plausible order of magnitude and gives the BO
optimizer a warm start rather than a random initialization.
\paragraph{Symbolic outer loop}
The LLM serves as the optimizer over the symbolic space $\Phi$.
After each batch of inner-loop optimization, the top-$m$ candidates by validation
score are collected into a reflection set $\mathcal{D}_{\text{top}}$.
Each candidate in $\mathcal{D}_{\text{top}}$ is summarized into structured
feedback comprising: the symbolic program $\phi_k$ with its optimized
parameters $\theta_k^*$; the surrogate scores on both $\mathcal{D}_{\text{train}}$
and $\mathcal{D}_{\text{val}}$; and a visualization of the potential trajectory
$p_t$ over the demonstration rollouts.
Given this feedback, the LLM identifies failure causes in the current
symbolic structures and revises the generated symbolic functions to improve logical
consistency and cross-trajectory invariance in the next batch.

\begin{figure*}[t]
\begin{center}
\includegraphics[width=0.92\linewidth]{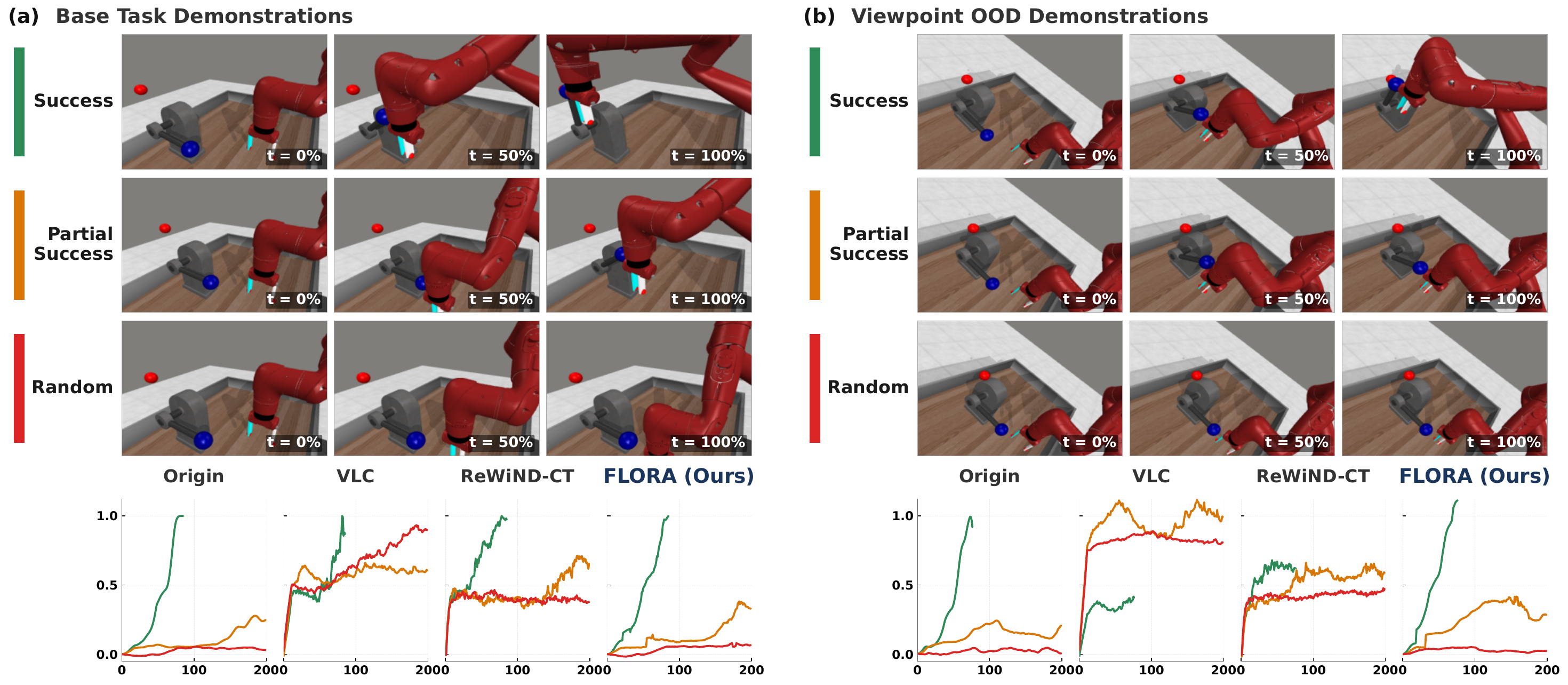}
\end{center}
\vspace{-0.5cm}
\caption{
Representative rollouts and reward curves on the Lever-Pull task under (a) base and (b) viewpoint-OOD settings.  Origin denotes the
default dense reward provided by the Meta-World simulator, serving as
the ground-truth reference. FLORA (Ours) maintains clear reward separation across success, partial-success, and random trajectories where baseline methods degrade.}
\label{fig: rewards image viewpoint OOD}
\vspace{-0.2cm}
\end{figure*}

\section{Experiments}
Experiments evaluate FLORA along four dimensions: reward quality
and OOD generalization in simulation (Sec. VI-A),
policy learning efficiency (Sec. VI-B),
real-world reliability 
(Sec. VI-C), and ablation analysis
(Sec. VI-D).

\subsection{Reward Quality and OOD Generalization}
 \begin{table}[t]
\centering
\caption{\textbf{Reward Evaluation Metrics.} Comparison of reward models in the training tasks and Position OOD Tasks, Viewpoint OOD and Object OOD Tasks.}
\vspace{-0.6em}
\small
\setlength{\tabcolsep}{4pt} 
\begin{tabular}{lccccc}
\toprule
\textbf{Metric} & \textbf{LIV} & \textbf{LIV-FT} & \textbf{VLC} & \textbf{ReWiND-CT} & \textbf{Ours} \\ 
\midrule
\multicolumn{6}{l}{\textbf{Training Task}} \\
\hspace{0.5em} Proc. Algn. $\rho\uparrow$  & -0.13 & 0.78 & 0.86 & 0.89 & \textbf{0.97} \\
\hspace{0.5em} Rew. RanK  $\uparrow$   & -0.03 & 0.42 & 0.01 & 0.53 & \textbf{0.80} \\
\hspace{0.5em} Rew. Diff. $\uparrow$      & -0.01 & 0.27 & 0.03 & 0.16 & \textbf{0.80} \\
\toprule
\multicolumn{6}{l}{\textbf{Position OOD}} \\
\hspace{0.5em} Proc. Algn. $\rho\uparrow$  & 0.13 & 0.55 & 0.65 & 0.61 & \textbf{0.85} \\
\hspace{0.5em} Rew. RanK  $\uparrow$   & -0.08 & 0.39 & 0.07 & 0.40 & \textbf{0.57} \\
\hspace{0.5em} Rew. Diff. $\uparrow$      & -0.01 & 0.17 & 0.10 & 0.03 & \textbf{0.46} \\
\midrule
\multicolumn{6}{l}{\textbf{Viewpoint OOD}} \\
\hspace{0.5em} Proc. Algn. $\rho\uparrow$  & 0.08 & 0.69 & 0.75 & 0.65 & \textbf{0.88} \\
\hspace{0.5em} Rew. RanK $ \uparrow$   & 0.07 & 0.38 & 0.10 & 0.41 & \textbf{0.67} \\
\hspace{0.5em} Rew. Diff. $\uparrow$      & 0.01 & 0.12 & 0.08 & 0.03 & \textbf{0.46} \\
\midrule
\multicolumn{6}{l}{\textbf{Object OOD}} \\
\hspace{0.5em} Proc. Algn. $\rho\uparrow$  & 0.04 & 0.67 & 0.61 & 0.67 & \textbf{0.81} \\
\hspace{0.5em} Rew. RanK $ \uparrow$   & 0.02 & 0.39 & 0.04 & 0.45 & \textbf{0.53} \\
\hspace{0.5em} Rew. Diff. $\uparrow$      & 0.00 & 0.16 & 0.02 & 0.04 & \textbf{0.32} \\
\bottomrule
\end{tabular}
\label{tab:evaluation_metrics}
\end{table}
Experiments are conducted on Meta-World~\cite{yu2020meta}, covering
four simple tasks shared with prior baselines for direct comparison,
and four multi-stage tasks that are among the most challenging in
Meta-World, to stress-test reward generalization beyond the
single- to two-substage settings evaluated in prior work.
For each task, the training set $\mathcal{D}_{\text{demo}}$ comprises  five  demonstrations  collected using a scripted policy in Meta-World for reward model learning. To evaluate generalization, we collect a test set $\mathcal{D}_{\text{test}}$ consisting of ten successful trajectories generated from novel random seeds, representing unseen spatial configurations.

Baselines include LIV~\cite{ma2023liv}, a vision-language reward
model pretrained on EpicKitchens~\cite{damen2020epic} and
fine-tuned on $\mathcal{D}_{\text{demo}}$ per task (LIV-FT);
Video-Language Critic (VLC) \citep{alakuijala2024video}, trained
on $\mathcal{D}_{\text{demo}}$; and ReWiND~\cite{zhang2025rewind},
pretrained on the Open X-Embodiment dataset $\mathcal{D}_{\text{open-x}}$ \citep{o2024open}  and a custom dataset containing 20 tasks with 5 demonstrations each. We evaluate a co-trained version ({ReWiND-CT}) that is trained on  both $\mathcal{D}_{\text{open-x}}$ and task-specific $\mathcal{D}_{\text{demo}}$ data for each task.

\paragraph{Process Alignment}
Process alignment measures how closely reward signals track actual task progression, quantified by Spearman's rank correlation against time on $\mathcal{D}_{\text{test}}$.
 As shown in Table \ref{tab:evaluation_metrics}, FLORA consistently outperforms all baselines in  process alignment. This superior performance on $\mathcal{D}_{\text{test}}$ demonstrates that FLORA  can successfully extrapolate to novel spatial configurations within the same task.

\paragraph{Policy Rollout Ranking and Discriminative Ability}
A robust reward model must not only award success but also accurately penalize failures. To test this discriminative ability, we construct an evaluation set by mixing: (i) 10 successful test trajectories ($\mathcal{D}_{\text{test}}$); (ii) 10 partial-success trajectories ($\mathcal{D}_{\text{PS}}$) generated by adding noise to the scripted policies; and (iii) 10 completely random trajectories ($\mathcal{D}_{\text{random}}$). We treat the  trajectory rank derived from Meta-World’s default dense rewards  as the ground-truth rank and  compare  rankings generated by  reward models against the  ground-truth rank using Spearman’s rank correlation. Additionally, we report the average  reward difference, which measures the numerical distance between the rewards assigned to successful vs. failed rollouts. 

As shown in Table \ref{tab:evaluation_metrics}, 
FLORA significantly outperforms the strongest baseline, achieving a 51\% improvement in rank correlation and a 196\% increase in reward margin. These results demonstrate FLORA’s superior ability to separate success from failure by focusing on semantic task progression rather than misleading visual artifacts.

\paragraph{OOD Generalization}
To evaluate the generalization ability of the reward models learned from  five demonstrations, we construct three distinct Out-of-Distribution (OOD) scenarios for each of the eight tasks: 
(i) {Position OOD}:  The initial positions of the object and gripper are shifted by 40\,cm, significantly exceeding the distribution of the training demonstrations; (ii) {Viewpoint OOD}:  The camera is rotated  by 30° relative to the reference viewpoint; (iii) {Object OOD}: The color and scale of task-relevant objects are modified.  All modifications are performed under the constraint that the robot gripper and objects remain visible to the camera.
 For each task variant, we construct corresponding $\mathcal{D}_{\text{test}}$, $\mathcal{D}_{\text{PS}}$ and $\mathcal{D}_{\text{random}}$ sets to evaluate the model’s progress alignment, policy ranking, and discriminative ability under distribution shift.

As  summarized in  Table \ref{tab:evaluation_metrics}, {FLORA} consistently outperforms all baseline models. While baseline reward models suffer significant performance degradation due to visual overfitting, {FLORA} maintains high process alignment and discriminative ability. 
To further assess reward quality, we visualize the reward curves produced by
each reward model for successful, partial-success, and random trajectories on 
the base and Viewpoint-OOD variants  of the Lever-Pull task.
As shown in Fig.~\ref{fig: rewards image viewpoint OOD}, 
while baselines such as LIV-FT and ReWiND fail to separate partial-success
from random trajectories under this distribution shift, FLORA maintains a
clear margin across all three trajectories.

\begin{figure}[t]
\begin{center}
\vspace{-0.2cm}
\includegraphics[width=0.8\linewidth]{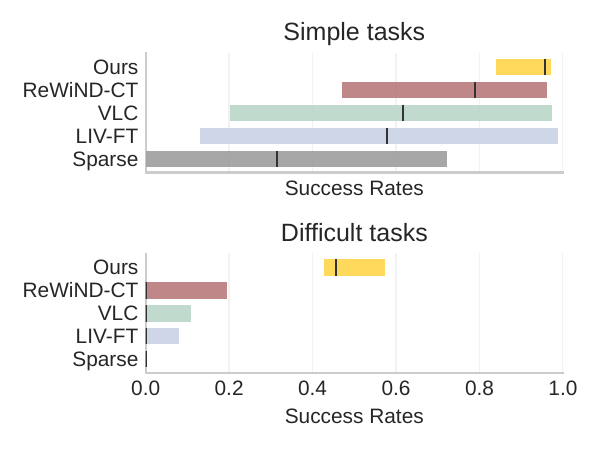}
\end{center}
\vspace{-0.7cm}
\caption{Meta-World Performance: We report interquartile means of success rates  with interquartile range in simple tasks and difficult tasks. }
\label{fig: SR simulation}
\vspace{-0.2cm}
\end{figure}

\subsection{Policy Learning Efficiency}
\paragraph{Default Tasks}
The reward model's effect on policy learning is assessed across eight Meta-World tasks, comprising four simple and four difficult tasks.
RLPD~\cite{ball2023efficient} serves as the base RL algorithm, maintaining identical hyperparameters across all experiments for a fair comparison. Baselines include LIV-FT, VLC, ReWiND-CT, and the default sparse binary signal. Each task is trained over three independent random seeds.  We report the Inter-Quartile Means (IQM) of success rates with interquartile range. For simple tasks, agents are trained for 100k timesteps, while agents on difficult tasks are trained for  200k timesteps. Detailed training curves  are provided in Supplementary Material III. 

As shown in Fig. \ref{fig: SR simulation}, FLORA significantly outperforms all  baselines. In simple tasks, FLORA achieves a 96\% success rate, representing a 22\% improvement over the strongest baseline. Crucially, in the  difficult tasks, FLORA maintains a 42\% success rate, nearly double the best baseline.  
These results are consistent with the reward quality metrics in Table~\ref{tab:evaluation_metrics}: methods with higher process alignment and ranking scores systematically yield faster policy convergence, validating our surrogate objective as a reliable predictor of  RL performance.

\begin{figure}[t]
\begin{center}
\includegraphics[width=0.85\linewidth]{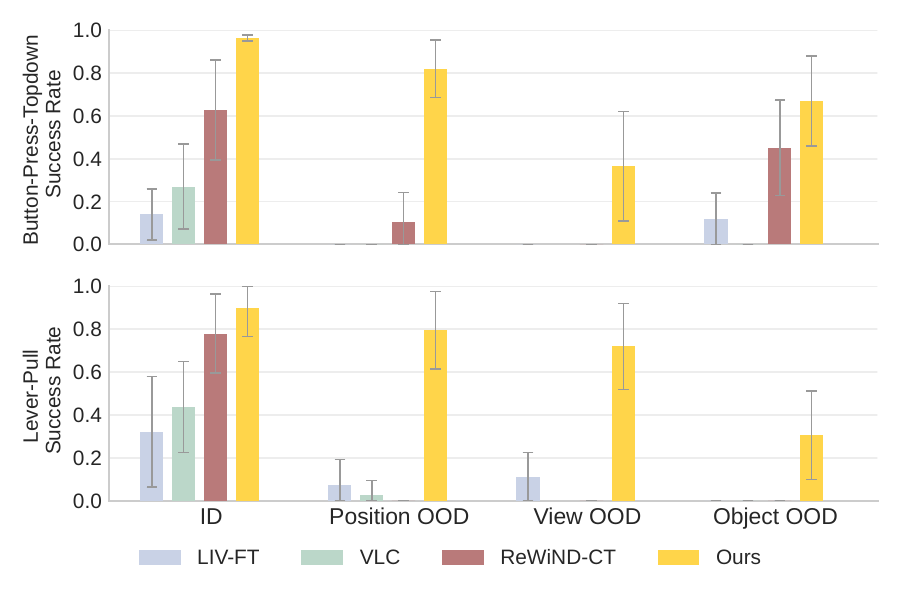}
\end{center}
\vspace{-0.7cm}
\caption{Generalization Performance: We reuse the trained reward models in three OOD tasks and report the average success rates across three
independent random seeds. }
\label{fig:OOD simulation}
\vspace{-0.3cm}
\end{figure}

\paragraph{OOD Tasks}
To further evaluate the generalization ability of our reward models, we reuse the trained reward models  in these three OOD tasks: Position OOD, Viewpoint OOD, and Object OOD. We select one simple (Button-Press-Topdown) task and one difficult task (Lever-Pull) as representative tasks for this analysis. The implementation details and learning curves are in the Supplementary Material III.
The results (Fig. \ref{fig:OOD simulation}) highlight the superior generalization ability of our reward model. While baseline reward models suffer from catastrophic failure in OOD settings due to visual distribution shifts, FLORA achieves high success rates in all OOD tasks.

\begin{figure}[h]
\begin{center}
\vspace{-0.5cm}
\includegraphics[width=0.95\linewidth]{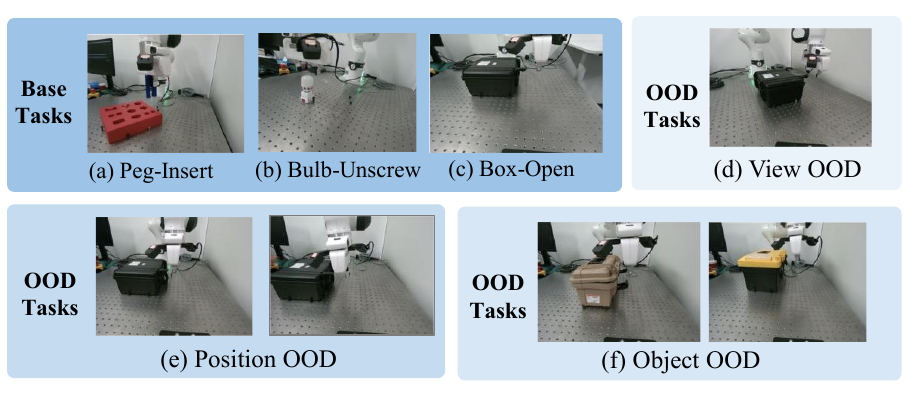}
\end{center}
\vspace{-0.7cm}
\caption{Real-world manipulation tasks and OOD variants. (a)–(c) Base tasks on the Franka arm. (d)–(f) OOD variants: viewpoint shift, position shift, and novel object instances.}
\label{fig: real world images}
\vspace{-0.6cm}
\end{figure}

\subsection{Real-world Experiments}
\paragraph{Policy Learning}
To evaluate the reliability of FLORA in the real world, we conduct experiments using a 7-DOF Franka Emika research arm \citep{haddadin2022franka}. We design three challenging manipulation tasks, as shown in Fig. \ref{fig: real world images}:
(i) {Peg-Insert}: Aligning and inserting a peg into a tight-tolerance hole, requiring millimeter-level precision and sustained contact.  
(ii) {Box-Open}:  Grasping and rotating a hinged lid to open the box, requiring stable contact and coordinated rotational motion.
(iii) {Bulb-Unscrew}: A complex, long-horizon task requiring the robot to grasp a bulb, execute multiple counterclockwise rotations to unscrew it, and successfully place it on the table.

Five demonstrations per task are provided to learn reward models.  As for policy initialization, we  provide 5 demonstrations for Peg-Insert and Box-Open, and 20 demonstrations for the long-horizon Bulb-Unscrew task to facilitate initial exploration.  We utilize RLPD \citep{ball2023efficient} as the base reinforcement learning algorithm. We  compare against  {ReWiND-CT}, the most powerful baseline in  simulation, and {Sparse Rewards}, a classifier-based reward signal $R \in \{0, 1\}$.  Full hardware details are provided in Supplementary Material IV.

As shown in Table \ref{tab:success_rates realworld}, FLORA exhibits superior sample efficiency and robustness: In the Peg-Insert and Box-Open tasks, FLORA reaches a near 100\% success rate within 20 minutes, doubling the learning speed of the strongest baseline ReWiND-CT.  In the complex Bulb-Unscrew task, FLORA is the only method capable of achieving success, reaching a 45\% success rate within one hour. All baselines fail to achieve a single successful rollout in this time budget.

\begin{table}[t]
\centering
\caption{Success rates  of three real-world experiments at 20 min, 40 min, and 60 min training time. }
\vspace{-0.6em}
\small
\begin{tabular}{llccc}
\toprule
\textbf{Time} & \textbf{Tasks} & Sparse &   ReWiND-CT &  Ours \\ 
\midrule
\multirow{3}{*}{{20 min}} & Box-Open & {0/20} & {8/20} & \textbf{20/20} \\
 & Peg-Insert & 0/20 &  5/20 & \textbf{20/20} \\
 & Bulb-Unscrew & \textbf{0/20} &  \textbf{0/20} & \textbf{0/20} \\
\midrule
\multirow{3}{*}{{40 min}} & Box-Open & {0/20} & \textbf{20/20} & \textbf{20/20} \\
 & Peg-Insert & 0/20 &  \textbf{20/20} & \textbf{20/20} \\
 & Bulb-Unscrew & 0/20 &  0/20 & \textbf{3/20} \\
\midrule
\multirow{3}{*}{{60 min}} & Box-Open & {0/20} & \textbf{20/20} & {20/20} \\
 & Peg-Insert & 8/20 &  \textbf{20/20} & \textbf{20/20} \\
 & Bulb-Unscrew & 0/20 &  0/20 & \textbf{9/20} \\

\bottomrule
\end{tabular}
\vspace{-0.4cm}
\label{tab:success_rates realworld}
\end{table}

 \begin{table}[h]
 \vspace{-0.3cm}
\centering
\caption{Success rates  of Box-Open task  at Position OOD, Viewpoint OOD and Object OOD at 30 min  training time.}
\vspace{-0.6em}
\small
\begin{tabular}{llccc}
\toprule
\textbf{} & \textbf{Tasks} & Sparse &   ReWiND-CT &  Ours \\ 
\midrule
\multirow{2}{*}{{Position OOD}} & position 1 & {0/20} & {10/20} & \textbf{20/20} \\
 &position 2  & 0/20 &  0/20 & \textbf{19/20} \\

\midrule
\multirow{1}{*}{{Viewpoint OOD}} & View1 & {0/20} & {2/20} & \textbf{15/20} \\

\midrule
\multirow{2}{*}{{Object OOD}} & Box1 & {0/20} & {0/20} & \textbf{18/20} \\
 & Box2 & 0/20 &  {0/20} & \textbf{17/20} \\
\bottomrule
\end{tabular}
\label{tab:success_rates OOD}
\end{table}

\paragraph{OOD Generalization Ability}
To evaluate the generality of our reward model in real-world experiments, we conduct experiments on three OOD task variants of the Box-Open task, as shown in  Fig. \ref{fig: real world images}:
 (i) {Position OOD}:
 Two levels of spatial perturbations relative to training: {Position 1} ($5 \times 5$\,cm shift) and {Position 2} ($10 \times 10$\,cm shift); (ii) {Viewpoint OOD}: A near 45-degree shift in camera extrinsics, significantly altering the visual perspective; (iii) {Object OOD}: Two novel box instances ({Box 1 and Box 2}) featuring different colors, textures, and scales, while sharing the same hinged mechanism. For each variant, we reuse the trained reward model and use it  to guide RL policy training for a fixed 30-minute budget.
 Table~\ref{tab:success_rates OOD} shows that FLORA maintains high success rates across all OOD scenarios. Notably, even under extreme viewpoint shifts and novel object geometries, FLORA enables the robot to achieve near-optimal performance within the 30-minute training budget. In contrast, the baseline models fail to provide accurate reward signals for novel scenarios, resulting in poor policy learning performance. 
 
 These results confirm that our method enables a single learned reward representation to  be reused in the diverse task variants in the real world.

\subsection{Ablation Study}

We conduct a series of ablation studies to  isolate the contribution of our core design choices.

\paragraph{Flow-Generator} 
The flow generator pipeline is evaluated along two axes: the contribution of object-centric normalization and the choice of point tracker. The full FLORA model is compared against a variant without the normalization module and a variant that substitutes the 3D tracker with a 2D alternative, CoTracker3~\cite{karaev2025cotracker3}.  Both process alignment and policy rollout ranking abilities are reported under nominal and positional OOD conditions.
The results in  Table \ref{tb: ablation} indicate that {object-centric normalization is one of the primary drivers of OOD generalization}. 
Furthermore, 3D point tracking outperforms the 2D alternative, as 3D motion flows align more naturally with the physical priors and semantic logic used by LLMs during reward learning.

\begin{table}[t]
\centering
\caption{\textsc{Ablation Study}: Subtracting ($-$) and adding ($+$) various FLORA components on Default and OOD tasks.}
\label{tab:ablation_narrow}
\vspace{-0.6em}
\newcolumntype{C}{>{\centering\arraybackslash}X}
\setlength{\tabcolsep}{2pt} 

\begin{tabularx}{\columnwidth}{lCCCC}
\toprule
\multirow{2}{*}{Model} & \multicolumn{2}{c}{(a) Default} & \multicolumn{2}{c}{(b) OOD-Position} \\
\cmidrule(lr){2-3} \cmidrule(lr){4-5}
 & \shortstack{Proc. Algn. \\ $\rho \uparrow$} & \shortstack{Rew. Rank \\ $ \uparrow$} & \shortstack{Proc. Algn. \\ $\rho \uparrow$} & \shortstack{Rew. Rank \\ $ \uparrow$} \\
\midrule
{Original FLORA} & {0.97} & 0.80 & {0.85} & {0.57} \\
\midrule
$-$ Object-Centric Norm & {0.97} & 0.78 & 0.85 & 0.47 \\

$-$ 3D Point Tracking   & 0.94   & 0.68   & 0.74   & 0.59   \\
$-$ Symbolic Function             & \textbf{0.99} & 0.60 & \textbf{0.96} & 0.52 \\
$-$ 3 demos             & {0.98} & 0.80 & 0.85 & 0.42 \\

$+$ 10 demos            & {0.98} & \textbf{0.84} & {0.88} & 0.58 \\
$+$ Human Subtask Labels & {0.97}   & \textbf{0.84}   & 0.85   & \textbf{0.68}   \\
\bottomrule
\end{tabularx}
\vspace{-0.3cm}
\label{tb: ablation}
\end{table}

\paragraph{Symbolic Function}  To evaluate the necessity of the symbolic potential function, the default setting is compared with a variant which replaces the symbolic potential function with a transformer network. The transformer network is trained   on the same motion-flow input with the same surrogate objective and evaluated under both ID and positional OOD conditions. The results in  Table \ref{tb: ablation} indicate that while 
the transformer network trivially interpolates the temporal progression of the 5 successful demonstrations, achieving a deceptive 0.99 Process Alignment), it catastrophically fails to penalize novel failure modes, dropping Reward Rank from $0.80$ to $0.60$ on ID. This proves that without the structural constraints of symbolic math, an unconstrained neural network cannot learn the strict boundaries of task failure from successful-only data and lacks the  structural inductive bias needed to extrapolate to novel configurations.
The symbolic function, by contrast, imposes task-relevant geometric constraints  that remain valid under distribution shift by construction. These results underscore that the symbolic potential is not merely a performance-enhancing component but a necessary structural prior for reliable OOD generalization.

\paragraph{Demonstration Numbers} 
To assess robustness to demonstration count, the default five-demonstration setting is compared against a low-data variant with three demonstrations and a higher-data variant with ten. Both process alignment and policy rollout ranking abilities are reported under nominal and positional OOD conditions.
The results in  Table \ref{tb: ablation} indicate that 
  performance gains  are marginal when going from 5 to 10 demonstrations, while reducing the number below five leads to a significant drop in generalization.

\paragraph{Subtask Labels}  
To assess the impact of annotation quality on reward synthesis, we compare our default VLM-generated subtask labels  against an oracle variant utilizing expert human labels. The VLM achieves a mean annotation accuracy of 84.7\% across the
task suite, validating its reliability as a proxy for human supervision.
 Potential functions are then generated for all eight Meta-World tasks using human labels and evaluated under both ID and positional OOD conditions.
The result is shown in  Table \ref{tb: ablation}.
Human labels yield a moderate performance boost, but at the cost of significant
manual effort; VLM-based labeling offers a practical, fully automatic
alternative with acceptable impact on performance.

\paragraph{PBRS-MS Module}
The contribution of the PBRS-MS module is isolated by comparing it against two ablated variants on the \textit{Lever Pull} and \textit{Peg Insert Side} tasks, both of which involve high-dimensional, multi-stage contact-rich manipulation. The two variants include Standard PBRS, which  applies a single unified potential across all task stages; and direct,  where $\Phi(s)$ is used as the reward signal  ($R = \Phi(s)$), rather than through the difference-based shaping mechanism.
As shown in Fig. \ref{fig:ablation rewards}, traditional PBRS suffers from training instability due to potential collapse in complex manipulation tasks. PBRS-MS significantly improves data efficiency and ensures stable convergence.

\begin{figure}[h]
\begin{center}
\vspace{-0.4cm}
\includegraphics[width=0.75\linewidth]{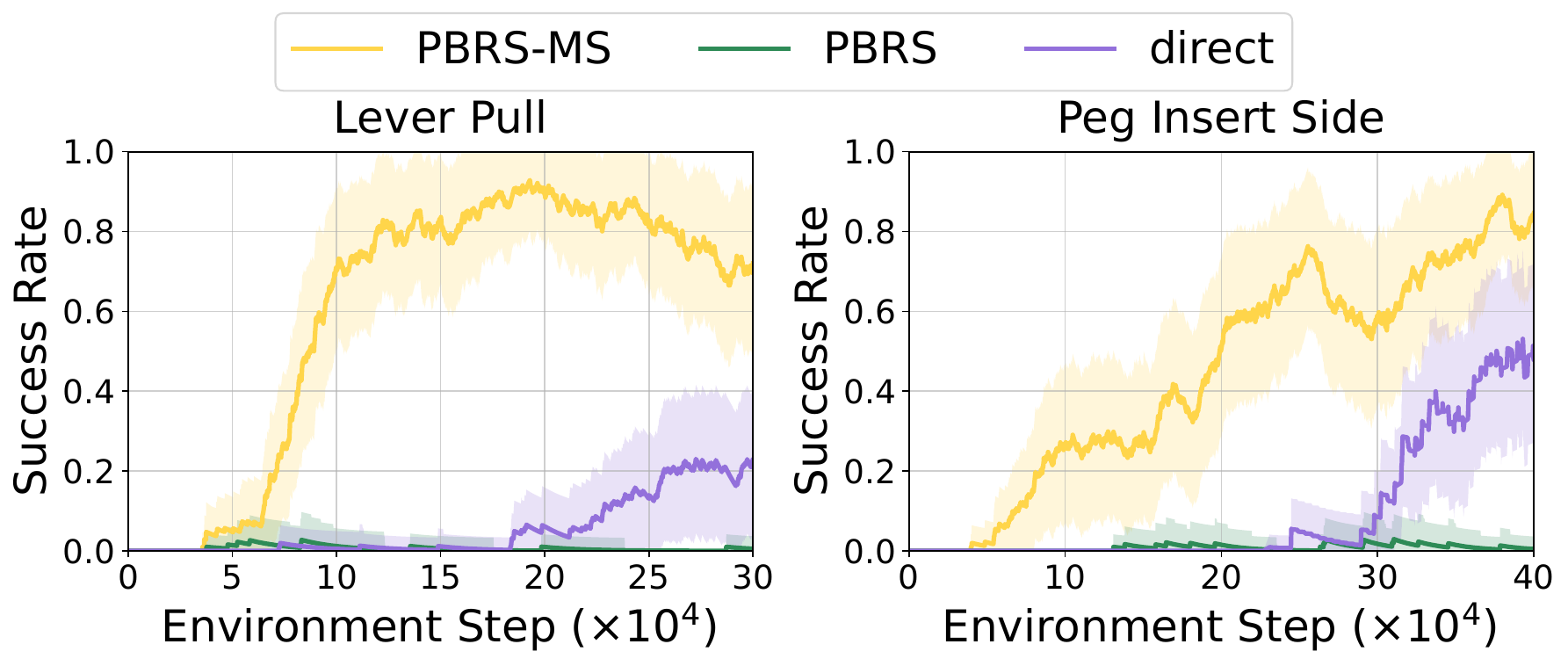}
\end{center}
\vspace{-0.5cm}
\caption{Ablation study on the PBRS-MS module.}
\label{fig:ablation rewards}
\vspace{-0.2cm}
\end{figure}

\begin{table}[h]
\centering
\vspace{-0.3cm}
\caption{Ablation on the optimization procedure.}
\label{tab:ablation_offline}
\vspace{-0.6em}
\setlength{\tabcolsep}{3pt}
\small
\begin{tabular}{@{}ll cccc@{}}
\toprule
Task & Variant & $C_\text{prog}$ & $C_\text{stage}$ & $C_\text{pbrs}$ & $\mathcal{J}$ \\
\midrule
\multirow{4}{*}{\textit{Lever-Pull}}
 & LLM Refl. + BO  & 0.92{\tiny$\pm$0.01} & \textbf{0.90}{\tiny$\pm$0.00} & \textbf{0.94}{\tiny$\pm$0.00} & \textbf{9.29}{\tiny$\pm$0.04} \\
 & LLM Refl. only  & 0.82{\tiny$\pm$0.21} & 0.68{\tiny$\pm$0.17} & 0.67{\tiny$\pm$0.14} & 7.31{\tiny$\pm$1.58} \\
 & BO only          & \textbf{0.96}{\tiny$\pm$0.01} & 0.78{\tiny$\pm$0.08} & 0.91{\tiny$\pm$0.02} & 9.14{\tiny$\pm$0.03} \\
 & Direct           & 0.44{\tiny$\pm$0.06} & 0.44{\tiny$\pm$0.00} & 0.58{\tiny$\pm$0.04} & 5.10{\tiny$\pm$0.06} \\
\midrule
\multirow{4}{*}{\textit{Peg-Insert}}
 & LLM Refl. + BO  & \textbf{0.90}{\tiny$\pm$0.05} & \textbf{0.88}{\tiny$\pm$0.01} & \textbf{0.89}{\tiny$\pm$0.07} & \textbf{8.95}{\tiny$\pm$0.15} \\
 & LLM Refl. only  & 0.85{\tiny$\pm$0.07} & 0.66{\tiny$\pm$0.01} & 0.73{\tiny$\pm$0.10} & 7.69{\tiny$\pm$0.26} \\
 & BO only          & 0.89{\tiny$\pm$0.06} & 0.75{\tiny$\pm$0.09} & 0.88{\tiny$\pm$0.08} & 8.74{\tiny$\pm$0.06} \\
 & Direct           & 0.84{\tiny$\pm$0.01} & 0.67{\tiny$\pm$0.01} & 0.58{\tiny$\pm$0.09} & 6.89{\tiny$\pm$0.47} \\
\bottomrule
\vspace{-0.3cm}
\end{tabular}
\end{table}

\paragraph{Hybrid Optimization Method}
The hybrid optimizer is compared against three reduced variants: LLM reflection alone, Bayesian Optimization alone, and direct selection of the best LLM-generated candidate without further optimization. Each variant is run five times; performance is measured by the three surrogate metrics $\mathcal{C}{\text{stage}}$, $\mathcal{C}{\text{prog}}$, and $\mathcal{C}_{\text{pbrs}}$, with results reported in Table~\ref{tab:ablation_offline}.
The results demonstrate that our hybrid approach, combining LLM Reflection with BO, achieves  superior performance.

These ablation  results show that object-centric normalization, 3D point tracking  and hybrid optimization improve performance and  OOD robustness, PBRS-MS ensures stable convergence; each component contributes to the overall performance.

\section{Discussion and Future Work}
The core claim of this work is that behavioral invariants sufficient
for real-world reward generalization can be distilled from as few as
five demonstrations, provided the reward representation is grounded
in object-centric motion rather than raw pixels and its structure is
constrained by symbolic inductive biases. The following sections
examine the evidence for this claim, characterize the boundaries of
the approach, and identify directions for future work.

\subsection{Sources of OOD Robustness}

We attribute FLORA's OOD robustness to three complementary inductive biases, each targeting a different source of visual overfitting. (1) {Input-level invariance.}
The Flow-Generator discards  appearance information and retains only the relative motion of task-relevant objects, so visually distinct scenes that share the same spatial topology produce
similar flow representations.
(2)
{Function-level invariance.}
Symbolic programs explicitly represent physical laws and logical constraints
that hold across task variants.  This restricted hypothesis space inherently prevents
the few-shot overfitting that plagues neural reward models trained on only
five demonstrations.
(3)
{Learning-level invariance.}
The  bi-level optimization  explicitly rewards potential functions that generalize to held-out validation trajectories, implicitly enforcing behavioral invariance without explicit domain randomization.

The ablation results corroborate this decomposition.
 Removing object-centric
normalization or substituting 2D for 3D point tracking degrades positional
OOD ranking (Table~\ref{tb: ablation}), confirming the contribution of
input-level invariance.   However, input-level invariance alone is insufficient: substituting the symbolic potential with a neural surrogate trained on the
same motion-flow input  immediately degrades OOD reward ranking (0.57 $\rightarrow$ 0.52) and ID discriminative ability (0.80 $\rightarrow$ 0.60), as empirically validated in Table~\ref{tb: ablation}.  
This confirms that the rigid geometric priors embedded in programmatic logic are the decisive factor preventing few-shot memorization.
Finally, the sharp
generalization collapse under reduced spatial coverage
(Table~\ref{tab:ablation_5 demonstrations}) indicates that cross-trajectory
diversity is essential for the optimization to discover
invariant programs.

\begin{table}[t]
\centering
\caption{Robustness of the 5-demonstration setting (Lever-Pull).}
\vspace{-0.6em}
\label{tab:ablation_5 demonstrations}
\setlength{\tabcolsep}{3pt}
\small
\begin{tabular}{@{}ll|ccc|ccc@{}}
\toprule
& & \multicolumn{3}{c|}{Default Task} & \multicolumn{3}{c}{Position OOD} \\
\cmidrule(lr){3-5} \cmidrule(lr){6-8}
Factor & Setting
& \shortstack{Proc.\\Algn.$\rho$\,$\uparrow$}
& \shortstack{Rew.\\Rank$\uparrow$}
& $\Delta$\scriptsize(\%)
& \shortstack{Proc.\\Algn.$\rho$\,$\uparrow$}
& \shortstack{Rew.\\Rank$\uparrow$}
& $\Delta$\scriptsize(\%) \\
\midrule
 & Ref & 1.00 & 0.89 & -- & 0.96 & 0.92 & -- \\
\midrule
\multirow{3}{*}{\textit{Seed}}
 & seed 1 & 1.00 & 0.89 & +0.1 & 0.99 & 0.92 & +1.3 \\
 & seed 2 & 0.99 & 0.88 & $-$1.0 & 0.92 & 0.94 & $-$0.9 \\
 & seed 3 & 1.00 & 0.91 & +0.9 & 0.97 & 0.91 & $-$0.4 \\
\midrule
\multirow{3}{*}{\textit{Cov.}}
 & Full & 1.00 & 0.88 & $-$0.6 & 0.96 & 0.92 & +0.1 \\
 & 1/2  & 1.00 & 0.92 & +1.5 & 0.89 & 0.71 & \textbf{$-$15.3} \\
 & 1/8  & 1.00 & 0.89 & +0.2 & 0.80 & 0.17 & \textbf{$-$48.6} \\
\midrule
\multirow{3}{*}{\textit{Qual.}}
 & 1 noisy & 1.00 & 0.89 & $-$0.1 & 0.97 & 0.90 & $-$0.4 \\
 & 3 noisy & 0.99 & 0.87 & $-$1.6 & 0.95 & 0.93 & $-$0.1 \\
 & 5 noisy & 1.00 & 0.75 & \textbf{$-$7.5} & 1.00 & 0.13 & \textbf{$-$40.0} \\
\bottomrule
\end{tabular}
\vspace{-0.3cm}
\end{table}

\subsection{Robustness of the Five-Demonstration Setting}
FLORA learns an invariant symbolic reward function from as few as five demonstrations. To quantify the robustness of this setting, we conduct ablation studies on the Lever-Pull task in Meta-World, examining seed choice, spatial coverage, and demonstration quality. Results are summarized in Table~\ref{tab:ablation_5 demonstrations}.

\paragraph{Seed Sensitivity}
The five demonstrations used in the main experiments were collected via Meta-World's scripted policy with a fixed random seed. Repeating data collection with three additional random seeds yields performance variations of at most $\pm$1.29\% on ID tasks and similarly small changes under positional OOD evaluation, confirming that the method's behavior is not an artifact of a particular seed choice.

\paragraph{Spatial Coverage}
We construct three demonstration sets that cover approximately the full task range, one-half, and one-eighth of the range. FLORA maintains strong ID performance even at 1/8 coverage. However, positional OOD performance decreases as coverage shrinks, suggesting that broad spatial variability is important for  OOD generalization.

\paragraph{Demonstration Quality}
Sensitivity to demonstration quality is assessed by injecting action
noise into the scripted policy to produce suboptimal trajectories, and 
varying their fraction across three settings.  The experimental results show that FLORA is resilient to imperfect data: even with a majority of suboptimal demonstrations (3 noisy + 2 clean), it maintains strong OOD performance. Performance collapses only when all five demonstrations are noisy, confirming that bi-level optimization effectively filters out suboptimal behaviors.

\subsection{Limitations and Future Work}

\paragraph{Monotonicity Assumption}

A core assumption of PBRS-MS is that task progress is monotonically
non-decreasing along the optimal trajectory.  This assumption is well-suited
to staged manipulation tasks.  However, it breaks down for tasks requiring \emph{reversible subtasks}.
For example, when a
peg jams during insertion, the robot must partially withdraw before
re-approaching at a corrected angle. The local PBRS term
$\gamma\phi(s_{t+1})-\phi(s_t)$ is negative during withdrawal, directly
penalizing the required backward motion, while the global milestone term
remains zero since $m_t$ is non-decreasing by construction. The combined
signal therefore discourages  optimal recovery strategies.  A similar
issue arises in tasks admitting multiple valid subtask orderings (e.g.,
assembling four table legs in any permutation): demonstrations following a
single ordering cause the learned potential to suppress alternative,
equally valid sequences. Extending PBRS-MS with mode-dependent or
order-invariant milestone formulations is a natural direction for future
work.

\paragraph{Perception and Computation}
As a vision-based method, FLORA is sensitive to occlusion: when the objects of interest are partially or fully occluded, the reward signal may degrade. Incorporating multimodal inputs such as proprioception or tactile feedback is a natural direction to improve robustness. Additionally, the current implementation requires a one-time, per-task-class specification of the objects of interest, which takes about one minute; as open-vocabulary detection models mature, this step could be fully automated.
Finally, the offline optimization stage currently takes 1--8 hours
per task, dominated by GPT-4.1 inference and BO optimization. As
reasoning-capable LLMs improve in code generation, this cost is
expected to fall substantially. In the longer term, distilling the
full pipeline into a lightweight end-to-end model would eliminate
the explicit symbolic search at deployment time.

\section{Conclusion} 
\label{sec:conclusion}

 We introduced FLORA, a framework for learning invariant symbolic reward functions from as few as five demonstrations. By combining object-centric motion flow representations, the PBRS-MS reward formulation, and hybrid LLM-BO optimization, FLORA successfully extracts task-level behavioral invariants from a few demonstrations. Across simulation and real-world manipulation experiments, FLORA provides dense and reliable reward signals, improves downstream policy learning, and enables zero-shot reuse across position, viewpoint, and object variations. These results suggest that symbolic, flow-grounded reward learning is a promising direction for building reusable reward functions for open-world robotic manipulation.

\bibliographystyle{IEEEtran}
\bibliography{references}

\vfill

\end{document}